\documentclass[12pt]{article}
\usepackage{amsfonts,amssymb,amsthm,amsmath,mathrsfs,bm}
\usepackage{graphicx}
\usepackage{enumerate}
\usepackage{algorithm,algorithmic}

\newtheorem{theorem}{Theorem}[section]
\newtheorem{lemma}[theorem]{Lemma}
\newtheorem{proposition}[theorem]{Proposition}

\theoremstyle{definition}
\newtheorem{definition}[theorem]{Definition}
\newtheorem*{remark}{Remark}
\newtheorem{example}{Example}

\renewcommand{\tilde}{\widetilde}
\newcommand{\hX}{\hat{X}}
\newcommand{\hC}{\hat{C}}
\renewcommand{\Pr}{\mathbb{P}}
\renewcommand{\Re}{\mathbb{R}}

\begin{document}
\title
{Perfect clustering for stochastic blockmodel graphs via adjacency spectral embedding}
\author{Vince Lyzinski$^{\dagger}$, Daniel L. Sussman$^{\ddagger}$, Minh Tang$^{\star}$,\\ Avanti Athreya$^{\star}$, Carey E. Priebe$^{\star}$\\
\small{$^{\dagger}$Johns Hopkins University Human
Language Technology Center of Excellence,}\\ 
\small{Baltimore, MD, USA}\\
\small{$^{\ddagger}$Department of Statistics,
Harvard University, Cambridge, MA, USA}\\
\small{$^{\star}$Johns Hopkins University, Baltimore, MD, USA}}
\date{}
\maketitle

\begin{abstract}
Vertex clustering in a stochastic blockmodel graph has wide
applicability and has been the subject of extensive research. In this
paper, we provide a short proof that the adjacency spectral embedding
can be used to obtain perfect clustering for the stochastic blockmodel
and the degree-corrected stochastic blockmodel. We also show an
analogous result for the more general random dot product graph model.
\end{abstract}



%

%
\maketitle
\section{Introduction}

In many problems arising in the natural sciences, technology, business
and politics, it is crucial to understand the specific connections
among the objects under study: for example, the interactions between
members of a political party; the firing of synapses in a neuronal
network; or citation patterns in reference literature.
Mathematically, these objects and their connections are modeled as
graphs, and a common goal is to find clusters of similar vertices
within a graph.

Both model-based and heuristic-based techniques have been proposed for
clustering the vertices in a graphs \cite{newman2006modularity,Bickel2009,Choi2010,Snijders1997Estimation}.
In this paper we focus on probabilistic performance guarantees for
spectral-based techniques which have elements of both model- and
heuristic-based methods \cite{rohe2011spectral,STFP-2011}.
We study the consistency of mean squared error clustering via the
adjacency spectral embedding for three nested classes of models, each
an examples of latent position models \cite{Hoff2002}:
\begin{itemize}
\item the stochastic blockmodel where vertices in the same cluster are
stochastically equivalent \cite{Holland1983},
\item the degree-corrected stochastic blockmodel where stochastic
equivalence holds up to a scaling factor \cite{karrer2011stochastic},
\item and the random dot product graph where a natural vertex
clustering may not exist \cite{young2007random}.
\end{itemize}
The generality of our main result allows for the extension of our
asymptotically error-free results from the rather restrictive
stochastic blockmodel to more general settings.

Numerous spectral clustering procedures have been proposed and analyzed
under various random graph models \cite
{chaudhuri12:_spect,rinaldo_2013,qin2013dcsbm,rohe2011spectral,STFP-2011}.
For example, Laplacian spectral embedding \cite{rohe2011spectral} and
adjacency spectral embedding \cite{STFP-2011} have been shown to yield
consistent clustering for the stochastic blockmodel. These results have
relied on bounding the Frobenius norm difference between the embedded
vertices and associated eigenvectors of the population Laplacian (in
Laplacian spectral embedding) or edge probability matrix (in adjacency
spectral embedding).

Relying on global Frobenius norm bounds for demonstrating consistent
clustering is suboptimal, however, because in general, one cannot rule
out that a diminishing but positive proportion of the embedded points
contribute disproportionately to the global error.
When this occurs, these ``outliers'' are very likely to be
misclustered, and hence the best existing bounds on the Frobenius norm
show that at most $O(\log(n))$ vertices will be
misclustered (see \cite[Theorem 3.1]{rohe2011spectral} and \cite[Theorem 1]{STFP-2011}).

In contrast, our main technical result gives a bound (in probability)
on the maximum error between {\it individual} embedded vertices and
the associated eigenvectors of the edge probability matrix (see
Lemma~\ref{lem:2toInf}). This lemma is proved for general random dot
product graphs and provides the necessary tools to improve the bounds
on the error rate of mean squared error clustering in adjacency
spectral embedding. The first main clustering result of this paper
gives a bound on the probability that a mean square error clustering
of the adjacency spectral embedding will be {\it error-free}, i.e.
zero vertices will be misclustered (see Theorem~\ref{t:clust}).

Due to the generality of our main lemma, we are able to prove an
analogous asymptotically error free clustering result in the
degree-corrected stochastic blockmodel (see Theorem \ref{t:clustDC}).
Note again that the best existing results for spectral methods in the
degree-corrected model assert that at most $O(\log(n))$ vertices will
be misclustered \cite[Theorem 4.4]{qin2013dcsbm}. Finally, we prove a
very general result that spectral clustering of random dot product
graphs is strongly universally consistent in the sense of
\cite{pollard81:_stron_k} (see Theorem~\ref{thm:univ}). These
extensions underly the wide utility of our approach, and we believe our
main lemma to be of independent interest.

We note that the authors of \cite{bickel2011method}, among others,
have shown that likelihood-based techniques can be employed to achieve
asymptotically error-free clustering in the stochastic blockmodel.
However, likelihood based approaches are computationally intractable
for very large graphs compared to our present spectral clustering
approach.

\section{Setting and main theorem}
In the first part of this section, we will define the random dot
product graph and our main tool, the adjacency spectral embedding.
Next, we define the stochastic blockmodel and clustering procedure, and
finally, we will state our main theorem and the supporting lemmas.

\subsection{Random dot product graphs and the adjacency spectral
embedding}
The random dot product graph model is a convenient theoretical tool,
and spectral properties of the adjacency matrix is well understood.
While the stochastic blockmodel relies on an inherently non-geometric
construction---indeed, each block in associated with a categorical
label, and these labels determine the adjacency probabilities---the
random dot product graph relies on a geometric construction in which
each block is associated with a point in Euclidean space, i.e. a
vector. The dot products of these vectors then determine the
adjacency probabilities in the graph.

\begin{definition}[Random Dot Product Graph (RDPG)]\label{def:rdpg}
A random adjacency matrix $A\sim\mathrm{RDPG}(X)$ for
$X=[X_1,\dotsc,X_n]^\top\in\mathcal{X}_d$ where
\begin{equation*}
\mathcal{X}_d=\{ Z\in\Re^{n\times d}: ZZ^\top\in[0,1]^{n\times n},
\mathrm{rank}(Z)=d\}
\end{equation*}
is said to be an instance of a random dot product graph (RDPG) if
\[
\Pr[A|X] = \prod_{i>j} (X_i^\top X_j)^{A_{ij}}{(1-X_i^\top X_j)}^{1-A_{ij}}.
\]
\end{definition}
\begin{remark}
In general we will denote the rows of an $n\times d$ matrix $M$ by
$M_i^\top$. With this notation, in the above definition
$P(A_{ij}=1)=X_i^{\top}X_j$ for all $1\leq i<j\leq n.$ We then
define $P=XX^\top$, so that the entries of $P$ give the Bernoulli
parameters for edge probabilities.
\end{remark}

Note that, as defined above, the rank of $P$ is $d$. Let
$P=[V|\widetilde V][S\oplus\widetilde S][V|\widetilde V]^T$ be the
spectral decomposition of $P$ where $[V|\widetilde V]\in\Re^{n\times
n}$ is orthogonal, $V\in\Re^{n\times d}$ has orthonormal columns,
$S\in\Re^{d\times d}$ is diagonal with
\begin{gather*}
S(1,1)\geq S(2,2)\geq\dotsb\geq S(d,d) > 0\text{ and }
\tilde{S} = 0.
\end{gather*}
Importantly, we shall assume throughout this paper that the non-zero
eigenvalues of $P$ are distinct, i.e., the inequalities above are
strict.

It follows that there exists an orthonormal $W\in\Re^{d \times d}$
such that $VS^{1/2}=XW$. We thus suppose that $X=V S^{1/2}$; the
assumption does not lead to any loss of generality because the
distribution of $A$ is invariant under orthogonal transformations of
the latent positions and the clustering method considered in this paper
is invariant under orthogonal transformations.
This relationship between the spectral decomposition of $P$ and the
latent positions $X$ for the RDPG model motivates our main tool: the
adjacency spectral embedding.

\begin{definition}[Adjacency Spectral Embedding (ASE)]\label{def:ase}
Let $\hat{V}\in\Re^{n\times d}$ have orthonormal columns given by
the eigenvectors of $A$ corresponding to the $d$ largest eigenvalues
of $A$ according to the algebraic ordering. Let $\hat{S}\in
\Re^{d\times d}$ be diagonal with diagonal entries given by these
eigenvalues in descending order. We define the $d$-dimensional {\em
adjacency spectral embedding} of $A$ via $\hX=\hat{V}
\hat{S}^{1/2}$.
\end{definition}

We shall assume, for ease of exposition, that the
diagonal entries of $\hat{S}$ are positive.
As will be seen later,
e.g., Lemma~\ref{lem:eigenvalues-concentration},
this assumption is justified in the context of random dot
product graphs due to the concentration of the eigenvalues of $A$
around those of $P$.
Recall that the rows of $\hX$ will be denoted by~$\hX_i^\top$.

\subsection{Clustering}

We begin by considering the task of clustering in the $K$-block
stochastic blockmodel. This model is typically parameterized by a
$K\times K$ matrix of probabilities of adjacencies between vertices in
each of the blocks along with the block memberships for each vertex.
Here we present an alternative definition in terms of the RDPG model.

\begin{definition}[(Positive Semidefinite) $K$-block Stochastic
Blockmodel (SBM)] We say an RDPG is an SBM with $K$ blocks if the
number of distinct rows in $X$ is $K$. In this case, we define the
block membership function $\tau:[n]\mapsto[K]$ to be a function
such that $\tau(i)=\tau(j)$ if and only if $X_i=X_j$. For each
$k\in[K]$, let $n_k$ be the number of vertices such that $\tau_i=k$,
i.e.\@ the number of vertices in block~$k$.
\end{definition}

To ease notation, we will always use $K$ to denote the number of
blocks in an SBM, and we will refer to a $K$-block SBM as simply an
SBM when appropriate.

\begin{remark}
Note that a general $K$-block SBM can only be
represented in this way if the $K \times K$ matrix of probabilities
is positive semidefinite.
\end{remark}

Next, we introduce mean square error clustering, which is the
clustering sought by $K$-means clustering.

\begin{definition}[Mean Square Error (MSE) Clustering]
The MSE clustering of the rows of $\hat{X}$ into $K$ blocks returns
\begin{gather*}
\hC:=\min_{C\in\mathcal{C}_K}\|C-\hX\|_F,\text{ where } \\
\mathcal{C}_K = \{C\in\Re^{n\times d}: C\text{ has }K\text{
distinct rows}\},\
\end{gather*}
are the optimal cluster centroids for the MSE clustering. We also
define the cluster membership function $\hat{\tau}:[n]\mapsto[K]$,
which satisfies $\hat{\tau}_i=\hat{\tau}_j$ if and only if
$\hat{C}_i=\hat{C}_j$, where $\hat{C}_{i}^\top$ is the $i^{\text{th}}$
row of $\hat{C}$.
\end{definition}

In our results, we consider the MSE clustering of the rows of
$\hat{X}$ in two contexts, the SBM and RDPG models defined above. We
will also consider a variation of the SBM, the degree-corrected SBM,
in which we perform MSE clustering on the rows of a projected version
of $\hat{X}$.

\subsection{Main theorems}
Before stating our main results, we indicate our notation for matrix
norms and we define constants to be used throughout the remaining
text. For a matrix $M\in\Re^{n \times d}$ we let $\|M\|_{2 \to
\infty} = \max_{i} \|M_i\|_2$, i.e. the maximum of the Euclidean
norm of the rows. For a square matrix $M\in\Re^{n \times n}$,
$\|M\|_2=\sqrt{\text{largest eigenvalue of }M^\top M}$ denotes the
spectral norm. The Frobenius norm of a matrix is denoted by
$\|\cdot\|_F$.

We define:
\begin{itemize}
\item$\Delta=\max_{i} \sum_{j \not= i} P_{ij}$ is the maximum of
the row sums of $P$;
\item$d$ is the rank of $P$;
\item$\gamma n=\min_{1 \leq i \leq d} |S(i+1,i+1) - S(i,i)|>0$ is
the minimum gap among the distinct eigenvalues of $P$;
\item$K$ is the number of blocks in the SBM.
\end{itemize}
Note, $\gamma n$ is not necessarily equal to the magnitude of the
smallest non-zero eigenvalue of $P$, as the gaps between consecutive
non-zero eigenvalues could be smaller. We now state a technical but
highly useful lemma in which we bound the maximum difference between
the rows of $\hat{X}$ and the rows of an orthogonal transformation of
$X$.
\begin{lemma}\label{lem:2toInf}

Suppose $0 < \eta< 1/2$ is given such that $\gamma n \geq4
\sqrt{\Delta\log{(n/\eta)}}$. Then, with probability at least
$1-2\eta$, one has
%
\begin{equation}\label{eq:2inf}
\|\hat{X}-X\|_{2\to\infty} \leq\frac{85 d \Delta^3
\log{(n/\eta)}}{(\gamma n)^{7/2}}.
\end{equation}
\end{lemma}

\begin{remark}
The parameter $\Delta$ is the maximum expected degree for the random
graph. In many models, $\gamma$ will be of the same order of the
density of the graph. If the density is very small, the bound in
the Eq.~\eqref{eq:2inf} will be large. Frequently $\Delta$ and
$\gamma n$ are of the same order, so that the bound in
Eq.~\eqref{eq:2inf} is of order $O(\log(n)/\sqrt{\gamma n})$ for $d$
fixed and $\eta$ decaying polynomially. See Example~\ref{ex:sbm}
for a simple illustration of how this bound can be applied.
Finally, note that we always have $\Delta>\gamma n$ so that the
condition in Lemma~\ref{lem:2toInf} implies that $\Delta> 16
\log(n/\eta)$, which allows us to use the results of
\cite{oliveira2009concentration,tropp2011freedman}.
\end{remark}

Lemma~\ref{lem:2toInf} gives far greater control of the errors than
the previous results that were derived for the Frobenius norm
$\|\hat{X}-X\|_F$; indeed, the latter bounds do not allow fine control
of the errors in the individual rows of $\hat{X}$, and therefore can
only bound the number of mis-clustered vertices via $O(\log
n)$. Lemma~\ref{lem:2toInf}, on the other hand, provides exactly this
control and, as such, vastly improves the bounds on the error rate of
MSE clustering of $\hat{X}$.

Our main theorem is the following result on the probability that
mean square error clustering on the rows of $\hat{X}$ is error-free.

\begin{theorem}[SBM]
\label{t:clust}
Let $A\sim\mathrm{RDPG}(X)$ be an SBM with $K$ blocks and
block membership function $\tau$ and suppose $\eta\in(0,1/2)$.
Assume that
\begin{enumerate}[{\bf(A0)}]
\item the non-zero eigenvalues of $P=XX^\top$ are distinct.
\end{enumerate}
Denote the bound on
$\|\hat{X}-X\|_{2\to\infty}$ in Lemma~\ref{lem:2toInf} as
$\beta=\beta(d,n,\eta,\gamma)$.
Let
$\hat\tau:[n]\rightarrow[K]$ be the optimal MSE clustering of the
rows of $\hat{X}$ into $K$ clusters. Let $S_K$ denote the symmetric
group on $K$, and $\pi\in S_K$ a permutation of the blocks. Finally,
let $n_{min}=\min_{k\in[K]} n_k$ be the smallest block size. If
\begin{enumerate}[\bf({A}1)]
\item for all $i,j\in[n]$ if $X_i\neq X_j$ then $\|X_i-X_j\|_2>6\beta
\sqrt{n/n_{min}}$ and
\item the eigenvalue gap satisfies $\gamma n >4\sqrt{\Delta\log
(n/\eta)},$
\end{enumerate}
then with probability at least $1-2\eta$,
\begin{equation*}
\min_{\pi\in S_K} |\{i\in[n]:\tau(i)\neq\pi(\hat{\tau}(i))\}|=0.
\end{equation*}
\end{theorem}

We remark that assumptions (A1) and (A2) are quite natural: (A1)
requires that the
rows of $X$ with distinct entries have some minimum separation
that is large enough compared to the ratio of the number of vertices to
the smallest block
size and compared to the bound in Lemma~\ref{lem:2toInf}. Assumption
(A2) on $\gamma$ ensures a large enough gap in the
eigenvalues to use Lemma~\ref{lem:2toInf}. We note that
Lemma~\ref{lem:2toInf} is applicable to the sparse setting, i.e., the
setting wherein the average degrees of the vertices are of order
$\omega(\log^{k}{n})$ for some $k \geq2$, but that we still need
sufficient separation between the distinct rows of $X$.
For a simple illustration of how this theorem can be applied for a
concrete model see Example~\ref{ex:sbm}.
Finally, we admit that assumption (A0) is less natural, but it is a
helpful technical restriction and it excludes a small range of parameters.

While Theorem \ref{t:clust} is proven in the SBM setting, we note
that our final theorem, Theorem~\ref{thm:univ}, is an analogous
clustering result in which we prove strong universal consistency of MSE
clustering for more general random dot product graphs.

Finally, we observe that
Theorem \ref{t:clust} has both finite-sample and asymptotic
implications. In
particular, under these model assumptions, for any finite $n$, the
theorem gives a lower bound on the probability of perfect
clustering. We do {\em not} assert---and indeed it is easy to
refute---that in the finite sample case, perfect clustering occurs
with probability one. Nevertheless, we can choose $\eta=n^{-c}$ for
some constant $c \geq2 $, in which case the probability of perfect
clustering approaches one as $n$ tends to infinity.

\begin{figure}[t]
\centering
\includegraphics[width=0.6\textwidth]{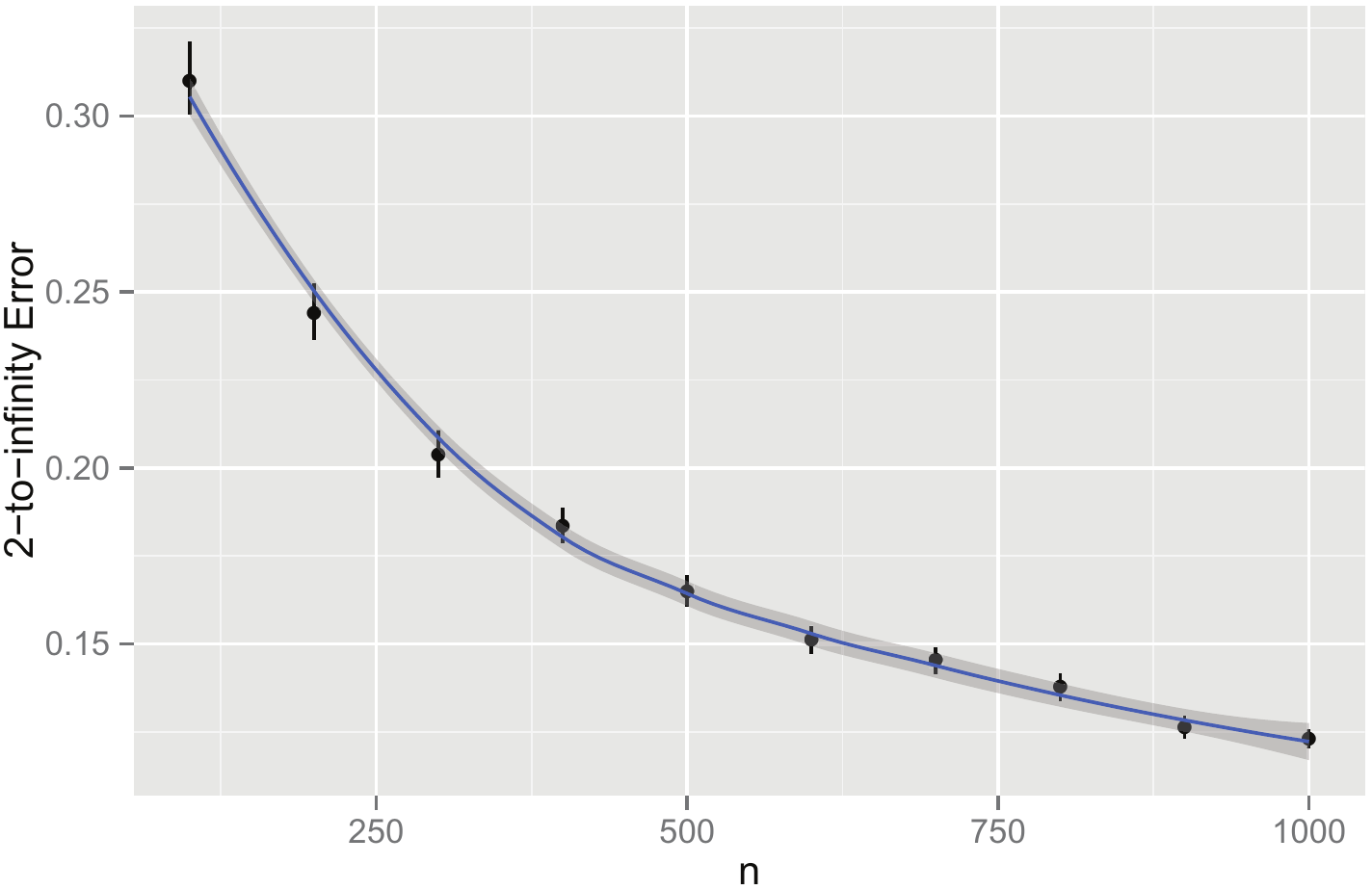}

\caption{Mean with standard error bars of $\|\hat{X}-X\|_{2\to\infty}$
for each value of $n$ for the model in Example~\ref{ex:sbm}.
The decay in the error is very close to $O(n^{-1/2})$.}
\label{fig:sbm}
\end{figure}

\begin{example}[Dense SBM]\label{ex:sbm}
Here we consider a simple concrete example where we can apply
Lemma~\ref{lem:2toInf} and Theorem~\ref{t:clust}.
Let $\nu_1 = (.5, .4)^\top$ and $\nu_2=(0.5,-0.4)^T$ and let
$X_i=\nu_1$ for $i=1,\dotsc,n/2$ and $X_i=\nu_2$ for $i=n/2+1,\dotsc,n$.
Hence, this is a two block model and the $2\times2$ matrix of edge
probabilities is given by
\[
B =
\begin{pmatrix}
0.41 & 0.09 \\
0.09 & 0.41
\end{pmatrix}
.
\]
The constants in our theorem are $d=2$, $\delta= 0.41n+0.09n=n/2$.
The distinct eigenvalues of $P$ are $n/2,0.32 n$, and 0 so the smallest
gap is $0.18n$; hence $\gamma=0.18$.
Lemma~\ref{lem:2toInf} can be applied as long as $0.18 n \geq4 \sqrt
{0.5 n \log(n/\eta)}$, which will clearly hold for $n$ sufficiently
large for any fixed $\eta\in(0,1/2)$.
This also establishes assumptions (A0) and (A1) of Theorem~\ref{t:clust}.

The implication of Lemma \ref{lem:2toInf} is that
\[
\| \hat{X}-X\|_{2\to\infty} \leq\frac{85\cdot2\cdot(0.5n)^3\log
(n/\eta)}{(0.18n)^{7/2}}\approx\frac{8588\log(n/\eta)}{\sqrt{n}}.
\]
While this bound is loose for small to moderate $n$, the
asymptotic implications are clear.
Empirically, Figure~\ref{fig:sbm} shows the average $\|\cdot\|_{2\to
\infty}$ error for this model as a function of $n$ and we see that the
error becomes small much sooner and decays at a rate very close to
$O(n^{-1/2})$.

For Theorem~\ref{t:clust} we can compute that if $X_i\neq X_j$ then $\|
X_i-X_j\|_2=0.8$ for all $n$ and hence since $n/n_{min}=2$ and $\beta
=O(\log(n)/\sqrt{n})$, the assumption (A1) will hold for $n$
sufficiently large.
Hence, for large enough $n$, there will be a very high probability that
the mean square error clustering will provide perfect performance.
\end{example}

\begin{example}[Sparse SBM]\label{ex:ssbm}
In this example, we will illustrate some asymptotic implications of the
assumptions of Theorem \ref{t:clust} in a generalization of
Example~\ref{ex:sbm}.
The SBM of the previous example is a specific instance of an SBM with
edge probabilities given by
%
\begin{equation}
B=c
\begin{pmatrix} a & b \\ b & a
\end{pmatrix}
,
\end{equation}
where $\tau_i=1$ for $i\in[n/2]$, and $\tau_i=2$ otherwise.
In order for $P$ to be positive semidefinite we need $a>b$, and the
subsequent RDPG representation is
%
\begin{equation}
\nu_1 = \frac{\sqrt{c}}{2}
\begin{pmatrix}
\sqrt{a+b}+\sqrt{a-b} \\ \sqrt{a+b}-\sqrt{a-b}
\end{pmatrix}
\text{ and } \nu_2 = \frac{\sqrt{c}}{2}
\begin{pmatrix}
\sqrt{a+b}-\sqrt{a-b} \\ \sqrt{a+b}+\sqrt{a-b}
\end{pmatrix}
.
\end{equation}
We will investigate for what values of $a$, $b$ and $c$ the assumptions
of our theorem are satisfied, noting first that (A0) is automatically satisfied.
For the key constants, we can work out that the distinct eigenvalues of
$P$ are $nc (a+b)/2$, $nc(a-b)/2,$ and $0$.
This gives $\gamma=c\min\{b,(a-b)/2\}$. In addition, $\Delta=
nc(a+b)/2$, so that assumption (A2) reads
\[
nc\min\{b,(a-b)/2\} > 4 \sqrt{nc(a+b) \log(n/\eta)/2}.
\]
We restrict ourselves to the sparse domain, and assume that $c=1/n$,
$a=o(n),$ and $b=o(n).$ Assumption (A2) then becomes
\[
\frac{\min\{b^2,(a-b)^2/4\}}{a+b} > 8\log(n/\eta).
\]
Here, we have $\|\nu_1-\nu_2\|_2=\sqrt{2(a-b)/n}$, and assumption
(A1) in this regime is equivalent to
\[
\frac{\sqrt{a-b}}{(a+b)^3}\left(\min\{b,(a-b)/2\}\right
)^{7/2}>127.5\sqrt{n}\log(n/\eta).
\]

Highlighting a few special cases, we consider
\begin{itemize}
\item[1.] $a=O(1)$ and $b=O(1)$. In this case, our assumptions do not
hold. Indeed, in \cite{mossel:ptrf} it is established that if
$(a-b)^2<2(a+b)$, then clustering is impossible, and the same authors
recently extended this result in \cite{mossel2014consistency} to show
that consistent estimation is impossible for any choice of $a=O(1)$ and
$b=O(1)$.
\item[2.] $b=\Theta(a)=\Theta(a-b)$. In order to satisfy assumptions
(A1) and (A2), it suffices that $b=\omega(\sqrt{n}\log(n)),$ and
(A1) does not hold if $b=o(\sqrt{n}\log(n))$.
\item[3.] $b=o(a).$ In order to satisfy assumptions (A1) and (A2), it
suffices that $b/\sqrt{a}=\omega(\sqrt{\log(n)})$ and
$b^{7/2}/a^{5/2}=\omega(\sqrt{n}\log(n)).$ Note that (A2) does not
hold if $b/\sqrt{a}=o(\sqrt{\log(n)}),$ and (A1) does not hold if
$b^{7/2}/a^{5/2}=o(\sqrt{n}\log(n)).$
\item[4.] $a-b=o(b).$ In order to satisfy assumptions (A1) and (A2),
it suffices that $a-b=\omega(\sqrt{\log(n)b})$ and $a-b=\omega
(n^{1/8}\log(n)^{1/4}b^{3/4}).$ Note that (A2) does not hold if
$a-b=o(\sqrt{\log(n)b})$, and (A1) does not hold if
$a-b=o(n^{1/8}\log(n)^{1/4}b^{3/4}).$
\end{itemize}
\end{example}

\section{Proof of Theorem~\ref{t:clust}}
Before we prove Theorem~\ref{t:clust}, we first collect a sequence of
useful bounds from
\cite{tang2012universally,tropp2011freedman,STFP-2011}. We then prove
two key lemmas.
\begin{proposition} \label{prop:oldBnd} Suppose
$A\sim\mathrm{RDPG}(X)$ with $X\in\mathcal{X}_d$ and let $\Delta$
and $\gamma$ be as defined in Lemma~\ref{lem:2toInf}. For any
$\eta\in(0,1/2)$, if $\gamma n > 4\sqrt{\Delta\log(n/\eta)}$,
then the
following occur with probability at least $1-\eta$
\begin{equation*}
\begin{aligned}
\|A-P\|_{2} &\leq2\sqrt{\Delta\log(n/\eta)},\\
\|\hat{V}-V\|_F^2
&\leq4d \,\, \frac{ \Delta\log(n/\eta)}{\gamma^2n^2}.
\end{aligned}
\end{equation*}
In addition, as $P$ is a non-negative matrix, $\|S\|_{2} \leq
\Delta$. Thus, if $\Delta\geq16 \log{(n/\eta)}$, then provided that
the above events occur, $\|\hat{S} \|_{2} \leq\min\{2 \Delta,n\}$.
\end{proposition}

The next two lemmas from \cite{athreya2013limit} are essential to our
argument.
\begin{lemma}[\cite{athreya2013limit}]
\label{lem:eigenvalues-concentration}
In the setting of Proposition~\ref{prop:oldBnd}, if the events in
Proposition~\ref{prop:oldBnd} occur, then
\begin{equation*}
\|\hat{S} - S\|_{2} \leq18 d \frac{\Delta^2
\log{(n/\eta)}}{\gamma^2 n^2}.
\end{equation*}
\end{lemma}

\begin{lemma}[\cite{athreya2013limit}]
\label{lem:1}
In the setting of Proposition~\ref{prop:oldBnd}, if the events in
Proposition~\ref{prop:oldBnd} occur, then
\begin{equation*}
\|V^\top\hat{V}-I\|_{F} \leq\frac{10d \Delta\log(n/\eta)}{
\gamma^2 n^2}.
\end{equation*}
\end{lemma}
We then have the following bound
\begin{lemma}
\label{lem:minh9}
In the setting of Proposition~\ref{prop:oldBnd}, if the events in
Proposition~\ref{prop:oldBnd} occur, then
\begin{equation*}
\|AV \hat{S}^{-1/2}-\hat{X}\|_F \leq\frac{24 \sqrt{2} d \Delta^2
\log(n/\eta)}{(\gamma n)^{5/2}}.
\end{equation*}
\end{lemma}

\begin{proof}
Let $E = A - \hat{V} \hat{S} \hat{V}^\top$. Denoting by $Z$ the
quantity $AV
\hat{S}^{-1/2}$, we have
\begin{equation*}
\begin{split}
\|Z - \hX\|_{F} &= \|A V \hat{S}^{-1/2}- \hat{V}
\hat{S}^{1/2}\|_{F} \\ &=
\|A(V - \hat{V})\hat{S}^{-1/2}\|_F \\ &
= \|( \hat{V} \hat{S} \hat{V}^\top+E)(V - \hat{V})\hat{S}^{-1/2}\|_F
\leq C_1 + C_2
\end{split}
\end{equation*}
where $C_1$ and $C_2$ are given by
\begin{gather*}
C_1 = \|\hat{S}\|_{2} \|\hat{V}^\top(V
- \hat{V})\|_F\|\hat{S}^{-1/2}\|_{2}; \quad 
C_2 = \|E\|_{2} \| V - \hat{V}\|_F \| \hat{S}^{-1/2}\|_{2}. %
\end{gather*}
Note that by Proposition~\ref{prop:oldBnd} and our assumption that
$\gamma\sqrt{n}\geq4\sqrt{\log(n/\eta)}$,
\begin{gather*}
\|\hat{S}^{-1}\|_{2}\leq(\gamma n-2\sqrt{n \log n/\eta})^{-1}\leq
\frac{2}{\gamma n}.
\end{gather*}
Combining the previous displayed equation and
Lemma~\ref{lem:1} yields
\begin{equation*}
C_1 \leq\frac{20 \sqrt{2} d \Delta^{2} \log(n/\eta)}{(\gamma n)^{5/2}}.
\end{equation*}
Similarly, we have that $\|E\|_{2\to2}
\leq 2\sqrt{\Delta\log(n/\eta)}, $ and combining this with
Proposition~\ref{prop:oldBnd}, we bound $C_2$ by 
\begin{equation*}
C_2
\leq\frac{4\sqrt{2} d \Delta\log(n/\eta)} {(\gamma n)^{3/2}},
\end{equation*}
from which the desired bound follows.
\end{proof}

We now use Lemma \ref{lem:eigenvalues-concentration}, Lemma \ref
{lem:minh9} and Hoeffding's
inequality to prove Lemma~\ref{lem:2toInf}.
We note that for any matrices $A\in\mathbb{R}^{k_1\times k_2}$ and
$B\in\mathbb{R}^{k_2\times k_2}$, $\|AB\|_{2\to\infty} \leq\|A\|_{
2\rightarrow\infty}
\|B\|_{2}$.
\begin{proof}[Proof of Lemma \ref{lem:2toInf}]
Since $X=P V S^{-1/2}$ we can
add and subtract the matrix $AV\hat{S}^{-1/2}$ and $AVS^{-1/2}$ to
rewrite $\hat{X} - X$ as
\begin{equation*}
\begin{split}
\hat{X}-X = A (\hat{V} - V) \hat{S}^{-1/2} + A V (\hat{S}^{-1/2} -
S^{-1/2}) + (A - P) V S^{-1/2}.
\end{split}
\end{equation*}
Lemma~\ref{lem:minh9} bounds the first term in terms of the Frobenius
norm which is a bound for the $2\to\infty$
norm. For the second term, we have
\begin{equation*}
\hat{S}^{-1/2} - S^{-1/2} = (S - \hat{S})(S^{1/2} +
\hat{S}^{1/2})^{-1} (\hat{S}^{-1/2} S^{-1/2}),
\end{equation*}
as both $\hat{S}$ and $S$ are diagonal matrices. Applying Lemma \ref
{lem:eigenvalues-concentration} thus yields
\begin{equation*}
\begin{split}
\|\hat{S}^{-1/2} -
S^{-1/2} \|_{2} & \leq\|S - \hat{S}\|_{2} (3/2
\sqrt{\gamma n}) (1/2 \gamma n) \\ &\leq24 d \frac{\Delta^2
\log{(n/\eta)}}{ (\gamma n)^{7/2}},
\end{split}
\end{equation*}
and hence
\begin{equation*}
\begin{split}
\|A V (\hat{S}^{-1/2} -
S^{-1/2}) \|_{2 \rightarrow\infty} & \leq\|AV\|_{2\to\infty}
\|\hat{S}^{-1/2} - S^{-1/2}\|_{2} \\ & \leq48 d
\frac{\Delta^3 \log{(n/\eta)}}{(\gamma n)^{7/2}}.
\end{split}
\end{equation*}
We now bound the third term.
Let $Z_{ij}$ denote the $i,j$th entry, and $Z_i^\top$ the $i$th row, of
the $n \times d$ matrix $(A-P)V$. Observe that
\begin{equation*}
\|(A-P)V\|_{2 \rightarrow\infty}=\max_{i}\|Z_i\|_2 \leq
\sqrt{d} \max_{i,j} |Z
_{i,j}|.
\end{equation*}
Next, since
\begin{equation*}
Z_{ij}=\sum_{k=1}^n (A_{ik}-P_{ik})(V)_{kj},
\end{equation*}
we see that $Z_{ij}$ is a sum of $n$ independent, mean zero random variables
$R_k(ij)=(A_{ik}-P_{ik})(V)_{kj}$, and $|R_k(ij)|\leq|(V)_{kj}|$.
Since $V$ has orthonormal columns, $\sum_{k}(V)^2_{kj}=1$.
Therefore, Hoeffding's inequality implies
\[
\Pr\Bigl(|Z_{ij}|>\sqrt{\tfrac12 \log(2nd/\eta)}\,\,
\Bigr) \leq\frac{\eta}{nd}.
\]
Since there are $nd$ entries $Z_{ij}$, a simple union bound ensures that
\[
\Pr\Bigl(\max_{i,j} |Z_{ij}|>\sqrt{\tfrac12
\log(2nd/\eta)}
\,\,\Bigr) \leq\eta,
\]
and consequently that
\begin{align*}
&\Pr\Bigl(\|(A-P)V\|_{2 \rightarrow\infty} >\sqrt{\tfrac{d}{2}
\log(2nd/\eta)}\,\,\Bigr)
\leq\eta.
\end{align*}
The third term can therefore be bounded as
\begin{equation*}
\begin{split}
\|(A - P) V S^{-1/2} \|_{2 \to\infty} \leq\|(A - P) V\|_{2 \to
\infty} \|S^{-1/2} \|_{2}
\leq\sqrt{\frac{d \log{(2nd/\eta)}}{2
\gamma n}}
\end{split}
\end{equation*}
with probability at least $1 - \eta$. Combining the bounds for the
above three terms yields Lemma~\ref{lem:2toInf}.
\end{proof}

\begin{proof}[Proof of Theorem~\ref{t:clust}]
Let $r:=\beta\sqrt{n/n_{min}}.$
We assume that the event in Lemma~\ref{lem:2toInf} occurs and show
that this implies the result. Since $X$ has $K$ distinct rows, it
follows that
\begin{equation*}
\|\hC-\hX\|_F\leq\|X -\hX\|_F\leq\beta\sqrt{n}.
\end{equation*}
Let $\mathcal{B}_1,\mathcal{B}_2,\ldots,\mathcal{B}_K$ be
$L^2$-balls with
radii $2r$ around the $K$ distinct rows of $X$. By the assumptions in
Theorem~\ref{t:clust}, these balls are
disjoint. Suppose there exists $k\in[K]$ such that $\mathcal{B}_k$ does not
contain any rows of $\hat{C}$.
Then $\|\hC- X \|_{F} > 2r \sqrt{n_{min}}$, as for each $k$, no row
of $\hat{C}$ is within $2r$ of the (at
least $n_{min}$) rows of $X$ in $\mathcal{B}_k$. This implies that
\begin{equation*}
\begin{split}
\|\hC- \hX\|_{F} &\geq\| \hC- X \|_{F} - \| \hX- X \|_{F} \\ &>
2r \sqrt{n_{min}} - \beta\sqrt{n} \\ &> 2 \beta
\sqrt{\frac{n}{n_{min}}} \sqrt{n_{min}} - \beta\sqrt{n} = \beta
\sqrt{n},
\end{split}
\end{equation*}
a contradiction. Therefore, $\|\hC- X\|_{2\to\infty}\leq2r$.
Hence, by the pigeonhole principle, each ball $\mathcal{B}_k$ contains
precisely one distinct row of $\hat{C}$.

If $X_i=X_j$, then both $\hat{C}_i$ and $\hat{C}_j$ are elements of
$\mathcal{B}_{\tau(i)}$, and since there is exactly one distinct row
of $\hat{C}$ in $\mathcal{B}_k$, $\hat{C}_i = \hat{C}_j$. Conversely,
if $\hat{C}_{i} \not= \hat{C}_j$, then $X_i$ and $X_j$ are in
disjoint balls $\mathcal{B}_k$ and $\mathcal{B}_{k'}$ for some $k,k'
\in[K]$, implying that
$X_i \not= X_j$. Thus, $X_i = X_j$ if and only if $\hC_i = \hC_j$,
proving the theorem.
\end{proof}

\section{Degree corrected SBM}
In this section we extend our results
to the degree corrected SBM \cite{karrer2011stochastic}.

\begin{definition}[Degree Corrected Stochastic Blockmodel (DCSBM)] We
say an RDPG is a DCSBM with $K$ blocks if there exist $K$ unit vectors
$y_1,\dotsc,y_K\in\Re^{d}$ such that for each $i\in[n]$, there exists
$k\in[K]$ and $c_i\in(0,1)$ such that $X_i=c_i y_k$.
\end{definition}

\begin{remark}This model is inherently more flexible than the standard
SBM because it allows for vertices within each block/community to have
different expected degrees. This flexibility has made it a popular
choice for modeling network data \cite{karrer2011stochastic}.
\end{remark}

For this model, we introduce $Y\in\Re^{n\times d}$ via
$Y_i^\top= y_{\tau(i)}^\top$, so that each row of $Y$ has unit
$L^2$-norm. As demonstrated in \cite{qin2013dcsbm}, a key to
spectrally clustering DCSBM graphs is to project the spectral embedding
onto the unit sphere, yielding an estimate of $Y$ rather than an
estimate of $X$. As such, let
$\hat{Y}=\mathrm{diag}(\hat{X}\hat{X}^\top)^{-1/2}\hat{X}$ where
$\mathrm{diag}(\cdot)$
denotes the operation of setting all the off-diagonal elements of the
argument to 0.
If we denote the unit sphere in $\Re^d$ by
$\mathcal{S}=\{x\in\Re^{d}:\|x\|_2=1\}$, then $\hat{Y}$ is the
projection of $\hat{X}$ on $\mathcal{S}$.
See Figure \ref{fig:dcsbm} for a simple example of this projection step.
%
\begin{figure}[t]
\centering
\includegraphics[width=0.75\textwidth]{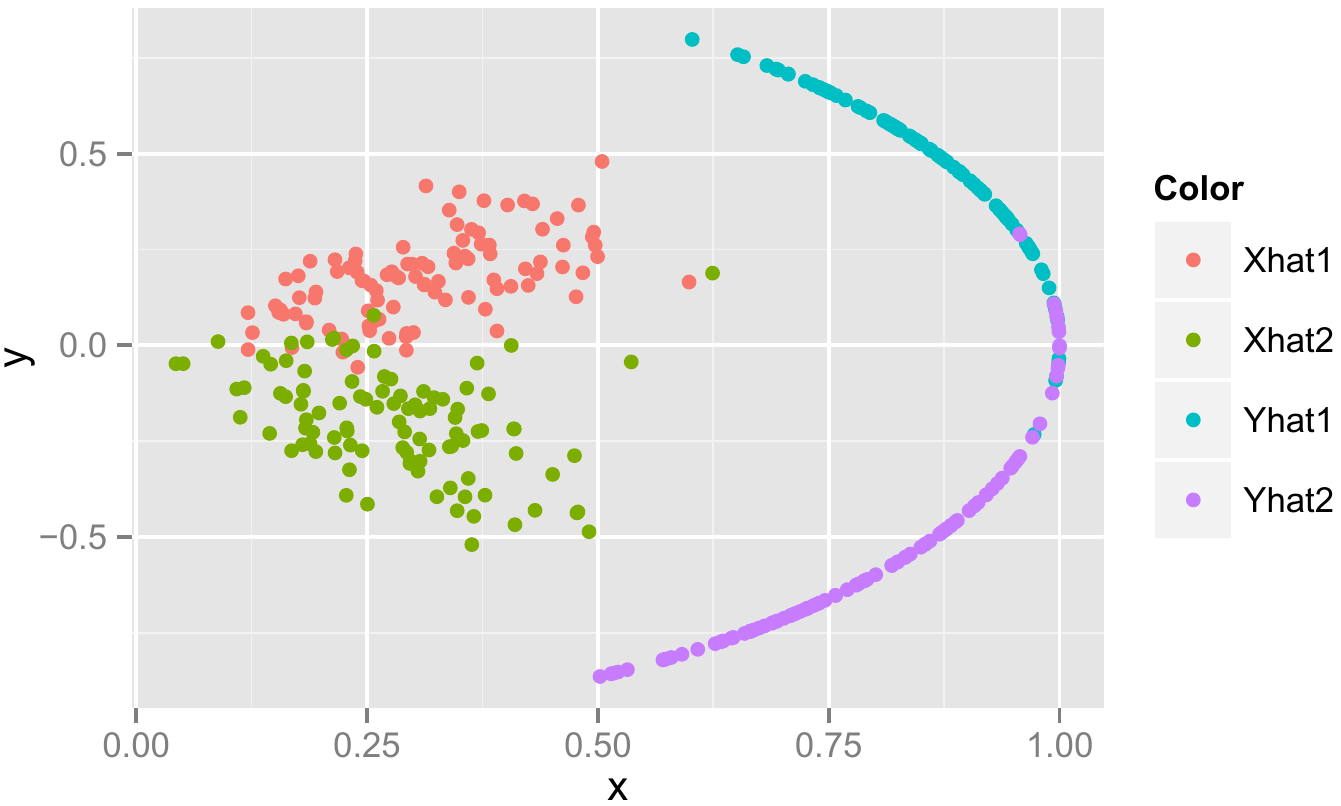}
\caption{Example of $\hX$ (the red and green points) and $\hat Y$
(the blue and purple points) for a 2-block DCSBM with latent positions
$y_1=[1/5,2\sqrt{6}/5]$ and $y_2=[2\sqrt{6}/5,1/5]$ and 100 vertices
in each block. The $c_i$'s are i.i.d.\@ Uniform(0.2,0.5).}
\label{fig:dcsbm}
\end{figure}

Our next lemma is the analogue of Lemma \ref{lem:2toInf} in the DCSBM
setting, allowing us to tightly control the errors in the individual
rows of $\hat Y$.
\begin{lemma}\label{lem:dcs2toInf} In the setting of
Lemma~\ref{lem:2toInf}, let the $n\times d$ matrices
$\tilde{Y},\hat{Y}\in\mathcal{S}$ be the projections of $X$ and
$\hat{X}$, respectively, onto $\mathcal{S}$. Let $c_{min}=
\min_{i\in[n]} \|X_i\|_2$. If $\|\hat{X}-X\|_{2\to\infty}\leq
\beta$, then
\[
\|\hat{Y}-\tilde{Y}\|_{2\to\infty} \leq\frac{2 \beta}{c_{min}}.
\]
\end{lemma}
\begin{proof}
We have $\tilde{Y}_i = (\|X_i\|_2)^{-1} X_i$ and $\hat{Y}_i =
(\|\hat{X}_i\|_2)^{-1} \hat{X}_i$. Straightforward calculations then yield
\begin{equation*}
\begin{split}
\|\tilde{Y}_i - \hat{Y}_i \|_2 &= \Bigl\| \frac{X_i}{\|X_i\|_2} -
\frac{\hat{X}_i}{\|\hat{X}_i\|_2} \Bigr\|_2 \\
&= \Bigl\|\frac{X_i}{\|X_i\|_2} - \frac{\hat{X}_i}{\|X_i\|_2} +
\frac{\hat{X}_i}{\|X_i\|_2} -
\frac{\hat{X}_i}{\|\hat{X}_i\|_2} \Bigr\|_2 \\
&\leq\frac{\|X_i - \hat{X}_i\|_2}{\|X_i\|_2} + \|\hat{X}_i\|_2
\Bigl(\frac{1}{\|X_i\|_2} - \frac{1}{\|\hat{X}_i\|_2} \Bigr)
\leq2 \frac{\|X_i - \hat{X}_i\|_2}{\|X_i\|_2} \leq\frac{2 \beta
}{c_{\mathrm{min}}}
\end{split}
\end{equation*}
as desired.
\end{proof}

As in Theorem \ref{t:clust}, this allows us to bound the probability of
error-free MSE clustering.
\begin{theorem}[Degree-corrected SBM]
\label{t:clustDC}
Suppose $A\sim\mathrm{RDPG}(X)$ and is a DCSBM with $K$ blocks and
block membership function $\tau$ and suppose $\eta\in(0,1/2)$. Let
$y_1,\dotsc,y_K$ be the $K$ unit vectors for the DCSBM and let
$c_{min}$ denote the smallest scaling factor. Let $\gamma,\beta$ be
as in Theorem~\ref{t:clust}. Suppose $r>0$ is such that for all
$i,j\in[K]$, $\|y_i-y_j\|_2>6r$. Let $\hat\tau:[n]\rightarrow
[K]$ be the optimal MSE clustering of the rows of $\hat{Y}$, the
projection of $\hat{X}$ onto $\mathcal{S}$, into $K$ clusters.
Finally, let $n_{min}=\min_{k\in[K]} n_k$ be the smallest block size.
If
\[
r> (2\beta\sqrt{n/n_{min}})/c_{min} \text{ and
}\gamma n>4\sqrt{\Delta\log(n/\eta)},
\]
then with probability at least
$1-2\eta$,
\[
\min_{\pi\in S_K} |\{i\in[n]:\tau(i)\neq\pi(\hat{\tau}(i))\}|=0.
\]
\end{theorem}
The proof of this theorem follows {\em mutatis mutandis} from the proof of
Theorem~\ref{t:clust}.

\section{Strong universal consistency}
\label{sec:strong-univ-cons}
We next show how our methodology can be used to prove strong universal
consistency of $K$-means clustering (as considered in \cite{pollard81:_stron_k}) in the general RDPG setting. Specifically,
suppose that $\{X_1, \dots,
X_n\}$ is a sample of independent observations from some common
compactly supported distribution
$F$ on $\mathbb{R}^{d}$. Denote by $F_n$ the empirical distribution of the
$\{X_i\}$, and let $C$ be a set containing $K$ or fewer
points. Suppose that $\phi\colon[0,\infty) \mapsto
[0, \infty)$ is a continuous, nondecreasing function with $\phi(0) =
0$. Now define $\Phi(C,F_n)$ and
$\Phi(C, F)$ by
\begin{gather*}
\Phi(C, F_n) = \int\bigl(\min_{c \in C} \phi(\|x - c\|)\bigr)
dF_n(x), \\
\Phi(C, F) = \int\bigl(\min_{c \in C} \phi(\|x - c\|) \bigr) dF(x).
\end{gather*}
The problem of $K$-means mean square error clustering given
$\{X_1,\dots,X_n\}$ can then be viewed as the minimization of $\Phi(A,
F_n)$ for $\phi(r) = r^2$ over all sets $C$ containing $K$ or fewer
elements. The strong consistency of $K$-means clustering corresponds
then to the following statement.
\begin{theorem}[\cite{pollard81:_stron_k}]
\label{thm:1}
Suppose that for each $k =
1,\dots, K$, there is a unique set $\bar{C}_k$ for which
\begin{equation*}
\Phi(\bar{C}_k, F) = \inf_{C \colon|C| = k} \Phi(C, F).
\end{equation*}
For any given $\{X_1,\dots,X_n\}$, denote by $C_n$ a minimizer
of $\Phi(C, F_n)$ over all sets $C$ containing $K$ or fewer
elements. Then $C_n \rightarrow\bar{C}_{K}$ almost surely and
$\Phi(C_n, F_n) \rightarrow\Phi(\bar{C}_K, F)$ almost surely.
\end{theorem}
We now state the counterpart to Theorem~\ref{thm:1} for the RDPG setting.
\begin{theorem}[RDPG]
\label{thm:univ}
Let $A\sim\mathrm{RDPG}(X)$, where the latent positions are sampled
from some common compactly supported distribution $F$. Let
$\hat{F}_n$ be the empirical distribution of the
$\{\hat{X}_i\}_{i=1}^{n}$. Denote by $\hat{C}_n$ a minimizer of
$\Phi(C, \hat{F}_n)$ over all sets $C$ containing $K$ or fewer
elements. Then provided that the conditions in Theorem~\ref{thm:1}
holds for $F$, $\hat{C}_n \rightarrow\bar{C}_K$ almost surely, and
furthermore, $\Phi(\hat{C}_n, \hat{F}_n) \rightarrow\Phi(\bar{C}_K,
F)$ almost surely.
\end{theorem}
\begin{proof}
We can suppose, without loss of
generality, that $F$ is a distribution on a totally bounded set, say
$\Omega$. Let $\mathcal{G}$
denote the family of functions of the form $g_{C}(x) = \min_{c \in
C} \phi(\|x - c\|_2)$ where $C$ ranges over all subsets of
$\Omega$ containing $K$ or fewer points. The theorem is
equivalent to showing
\begin{equation*}
\sup_{g \in\mathcal{G}} \Bigl| \int g \,\, d\hat{F}_n - \int g \,\,dF
\Bigr| \overset{\mathrm{a.s.}}{\longrightarrow} 0.
\end{equation*}
By Theorem~\ref{thm:1}, we know that
\begin{equation*}
\sup_{g \in\mathcal{G}} \Bigl| \int g \,\, d F_n - \int g \,\,dF
\Bigr| \overset{\mathrm{a.s.}}{\longrightarrow} 0.
\end{equation*}
and so the theorem holds provided that
\begin{equation*}
\sup_{g \in\mathcal{G}} \Bigl| \int g \,\, d\hat{F}_n - \int g \,\,dF_n
\Bigr| =
\sup_{g \in\mathcal{G}} \Bigl| \frac{1}{n} \sum_{i=1}^{n} \bigl\{
\min_{c
\in C} \phi(\|\hat{X}_i - c\|_2) - \min_{c \in C} \phi(\|X_i -
c\|_2) \bigr\} \Bigr| \overset{\mathrm{a.s.}}{\longrightarrow} 0.
\end{equation*}
Let $s_i$ denote the summand in the above display. We then have the
following bound
\begin{gather*}
|s_i| \leq\max_{c \in C}\{\bigl|\phi(\|\hat{X}_i - c\|_2) - \phi
(\|X_i -
c\|_2) \bigr| \};
\end{gather*}
and hence
\begin{equation*}
\begin{split}
\Bigl|\frac{1}{n}\sum_{i=1}^{n} s_i\Bigr| &\leq\frac{1}{n}\sum
_{i=1}^{n} \sum_{c \in C} |\phi(\|\hat{X}_i - c\|_2) - \phi(\|X_i -
c\|_2) | \\
& \leq K \max_{i} \max_{c \in C} |\phi(\|\hat{X}_i - c\|_2) - \phi
(\|X_i -
c\|_2) |.
\end{split}
\end{equation*}
We thus have the bound
\begin{equation*}
\sup_{g \in\mathcal{G}} \Bigl|\frac{1}{n}\sum_{i=1}^{n} s_i\Bigr|
\leq K \sup_{c \in\Omega} \max_{i} |\phi(\|\hat{X}_i - c\|_2) -
\phi(\|X_i -
c\|_2) |.
\end{equation*}
Now, by Lemma~\ref{lem:2toInf}, $\sup_{i}\|\hat{X}_i - X_i\|_2$
converges to $0$ almost
surely. Since $\phi$ is continuous on a compact set,
it is uniformly continuous. Thus
\begin{equation*}
\sup_{c \in\Omega}\sup_{i}
|\phi(\|\hat{X}_i - c\|_2) - \phi(\|X_i - c\|_2) |
\overset{\mathrm{a.s.}}{\longrightarrow} 0
\end{equation*}
as desired.
\end{proof}

\section{Discussion}
Lemma~\ref{lem:2toInf} provides a bound on the $2$-to-$\infty$ norm
of the
difference between $\hat{X}$ and $X$. The ability to control the
errors of individual rows of $\hat X$ allows us to prove asymptotically
almost surely perfect clustering in the SBM and DCSBM, a substantive
improvement on the best existing spectral clustering results.

Our approach can be easily modified to prove several extensions of
Theorem \ref{t:clust}. For example, we can consider the
special case in which the constants are fixed in $n$ and
$n_{min}=\Theta(n)$, whereupon the conditions of Theorem~\ref{t:clust}
are all satisfied for $n$ sufficiently large. In this setting, we can
further suppose that there are $K$ positions $\xi_1,\dotsc, \xi_k$ and
$\Pr[X_i=\xi_k]=\pi_k$ for some $\pi_k>0$, i.e. a mixture of point
masses. In other words, this is an stochastic block model with independent,
identically distributed block memberships (that are not fixed) across
vertices. Proving that the number of errors converges almost surely to
zero is then an easy application of Theorem~\ref{t:clust}.
Furthermore, our methods can be extended to alternate clustering
procedures, such as Gaussian mixture modeling (see \cite{suwan14:_empbayes}) or hierarchical clustering.

Indeed, one can construct many examples where perfect performance
is\break
achieved asymptotically (see Examples~\ref{ex:sbm} and \ref{ex:ssbm}).
We will not detail all regimes explicitly,
but rather note that this theory can be easily applied to handle a
growing number
of blocks, possibly impacting $d$, $\gamma$ and $n_{min}$, and
moderately sparse regimes, impacting $\gamma$.

We believe Lemma \ref{lem:2toInf} to be of independent interest apart
from clustering. Indeed, Lemma \ref{lem:2toInf} is a key result in
proving consistency of a divide-and-conquer seeded graph matching
procedure \cite{lyzinski:_seeded}. The lemma also leads to an easy
proof of the strong consistency of $k$-nearest-neighbors for vertex
classification, thereby extending the results of
\cite{sussman2012universally}. The lemma is a key component of the
construction of a consistent two-sample graph hypothesis test
\cite{tang14:_two}. Additionally, we are exploring the
implications of the lemma on parameter estimation for more general
latent position random graphs.\vadjust{\goodbreak}

The DCSBM is inherently more general than the SBM, and has key properties
useful in modeling group structures in graphs. In
\cite{qin2013dcsbm}, the authors provide complementary results for
spectral analysis of the DCSBM without requiring lower bounds on the
degrees; however, in turn, they obtain less-than-perfect clustering.
Our results are the first to show that, depending on model parameters,
the probability of perfect clustering tends to one as the number of
vertices tends to infinity. The keys to the easy extension of these
results to more general models are Lemmas~\ref{lem:2toInf} and
\ref{lem:minh9}, stated here in the RDPG setting.

For a general RDPG, there may not be a ``natural'' community structure.
Nonetheless, the strong universal consistency result of Theorem~\ref
{thm:univ} ensures that clustering the embedded graph will be
asymptotically equivalent to the clustering the true latent positions.
Finding the $k$-centers of the estimated latent positions provides one
way to approximate the distribution of the latent positions as a
mixture of point masses corresponding to an SBM, where the $k$ distinct
latent positions are given by the $k$-centers.
Approximating a more general graph distribution as a stochastic
blockmodel has been studied by \cite{wolfe13:_nonpar} and \cite{choi2014co}, and here we have detailed one spectral solution to this problem.
If $k$ is chosen appropriately, these approximations yield suitably
parsimonious distributions that can be used for understanding large
complex graphs, without requiring the estimation of a correspondingly
complex distribution.

\section*{Acknowledgments}
This work is partially supported by a
National Security Science and Engineering Faculty
Fellowship (NSSEFF), Johns Hopkins University Human Language
Technology Center of Excellence, and the
XDATA program of the Defense Advanced Research Projects
Agency. Lastly, we would like to thank the anonymous referees for their
helpful comments and suggestions and for suggesting Example~\ref{ex:ssbm}.

\bibliographystyle{plain}
\bibliography{biblio}
\end{document}